\pdfoutput=1

\documentclass[11pt]{article}

\usepackage[]{EMNLP2022}

\usepackage{times}
\usepackage{latexsym}
\usepackage{hyperref}

\usepackage[T1]{fontenc}
\usepackage{graphicx}
\usepackage{pifont}
\newcommand{\cmark}{\ding{51}}
\newcommand{\xmark}{\ding{55}}

\usepackage[utf8]{inputenc}

\usepackage{amsmath}
\usepackage{amssymb}

\usepackage{booktabs}
\usepackage{multirow, multicol}
\usepackage{microtype}
\usepackage{cleveref}

\crefformat{section}{\S#2#1#3}
\crefformat{subsection}{\S#2#1#3}

\usepackage{amsfonts}
\usepackage{subfig}
\usepackage{enumitem}
\usepackage{amsthm}
\usepackage{soul}

\usepackage[ruled,vlined,linesnumbered]{algorithm2e}
\newcommand{\modelname}{MM-Align}

\newcommand{\declarelogo}[0]{\includegraphics[height=.02\textwidth]{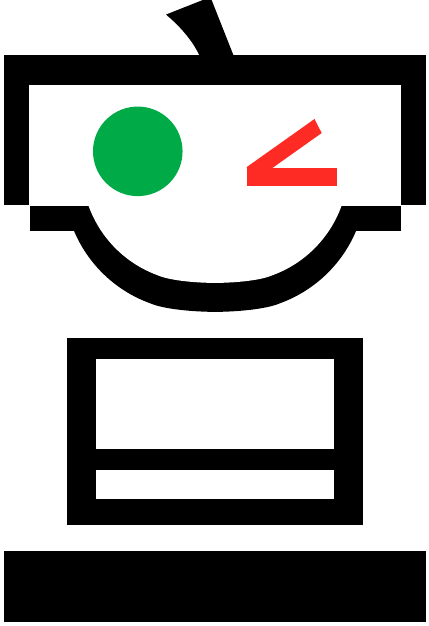}}

\title{MM-Align: Learning Optimal Transport-based Alignment Dynamics for Fast and Accurate Inference on Missing Modality Sequences}

\author{
    Wei Han\textsuperscript{\declarelogo} \quad
    Hui Chen\textsuperscript{\declarelogo} \quad
    Min-Yen Kan\textsuperscript{$\clubsuit$} \quad
    Soujanya Poria\textsuperscript{\declarelogo} \\
    \textsuperscript{\declarelogo} DeCLaRelab, Singapore University of Technology and Design, Singapore \\
    \textsuperscript{$\clubsuit$} National University of Singapore, Singapore \\
    \texttt{\{wei\_han,hui\_chen\}@mymail.sutd.edu.sg} \\
    \texttt{kanmy@comp.nus.edu.sg,sporia@sutd.edu.sg}
}

\begin{document}
\maketitle
\begin{abstract}
    Existing multimodal tasks mostly target at the \textit{complete input modality} setting, i.e., each modality is either \textit{complete} or \textit{completely missing} in both training and test sets.
    However, the randomly missing situations have still been underexplored.
    In this paper, we present a novel approach named~\modelname~to address the missing-modality inference problem.
    Concretely, we propose 1) an alignment dynamics learning module based on the theory of optimal transport (OT) for indirect missing data imputation; 2) a denoising training algorithm to simultaneously enhance the imputation results and backbone network performance.
    Compared with previous methods which devote to reconstructing the missing inputs,~\modelname~learns to capture and imitate the alignment dynamics between modality sequences.
    Results of comprehensive experiments on three datasets covering two multimodal tasks empirically demonstrate that our method can perform more accurate and faster inference and relieve overfitting under various missing conditions. 
    Our code is available at \url{https://github.com/declare-lab/MM-Align}.
\end{abstract}
\section{Introduction}
The topic of multimodal learning has grown unprecedentedly prevalent in recent years~\citep{ramachandram2017deep,baltruvsaitis2018multimodal}, ranging from a variety of machine learning tasks such as computer vision~\cite{zhu2017toward, nam2017dual}, natural langauge processing~\cite{fei2021cross, ilharco2021recognizing}, autonomous driving~\cite{caesar2020nuscenes} and medical care~\cite{nascita2021xai}, etc.
Despite the promising achievements in these fields, most of existent approaches assume a \textit{complete input modality} setting of training data, in which every modality is either \textit{complete} or \textit{completely missing} (at inference time) in both training and test sets~\citep{pham2019found,tang-etal-2021-ctfn, zhao-etal-2021-missing}, as shown in Fig.~\ref{complete}~and~\ref{completelymissing}.
Such synergies between train and test sets in the modality input patterns are usually far from the realistic scenario where there is \textit{a certain portion} of data without parallel modality  sequences, probably due to noise pollution during collecting and preprocessing time. 
\begin{figure}[t]
\newcommand{\figah}{6cm}
\newcommand{\sca}{0.32}
\makeatletter
\newcommand{\definetrim}[2]{%
  \define@key{Gin}{#1}[]{\setkeys{Gin}{trim=#2}}%
}
\makeatother
\definetrim{trimp}{1.7cm 0.7cm 1.0cm 1.0cm}
\definetrim{trimpr}{1.0cm 0.7cm 1.7cm 1.0cm}
\centering
\subfloat[\label{complete}]{\includegraphics[page=1,scale=\sca, trimp]{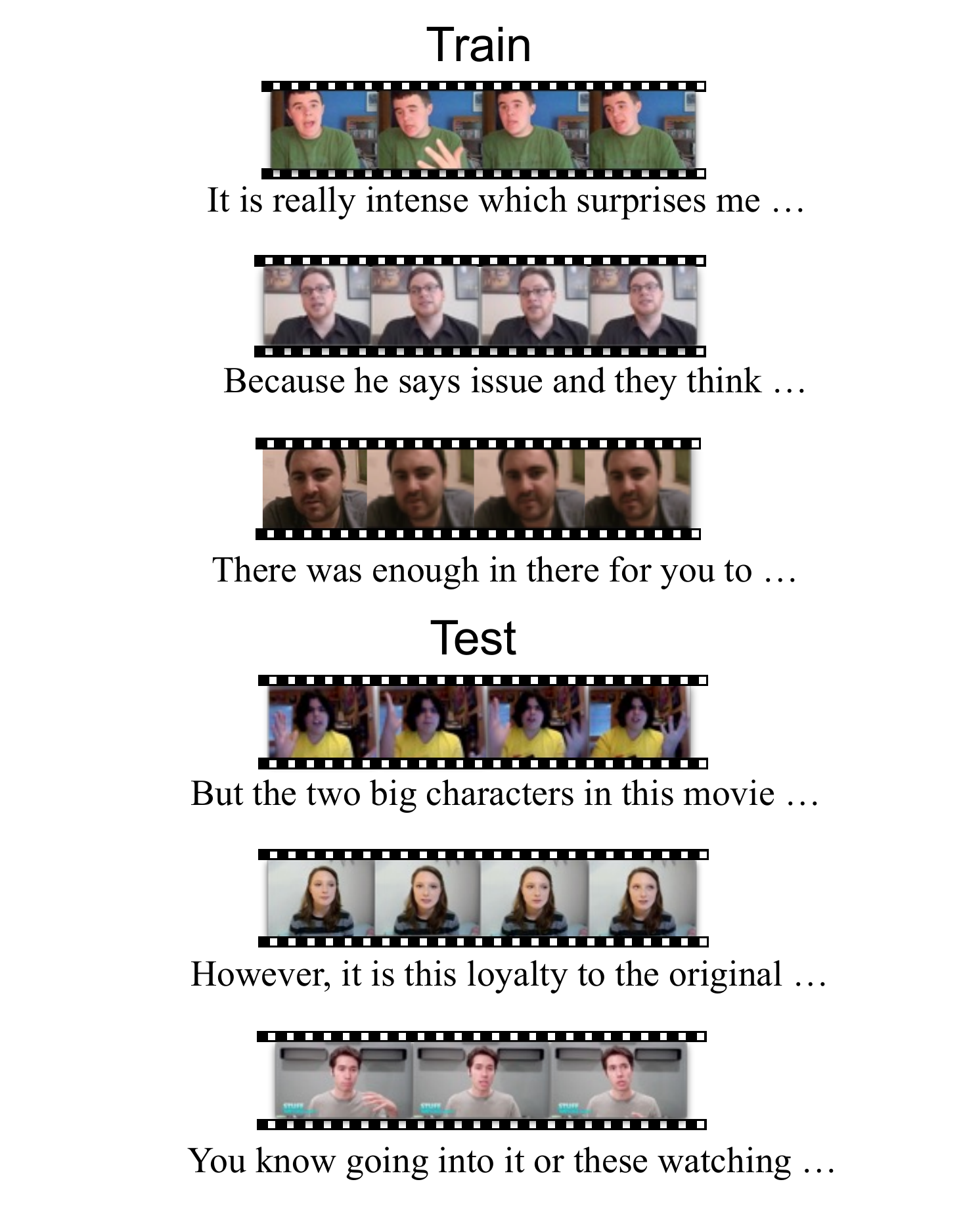}}
\subfloat[\label{completelymissing}]{\includegraphics[page=2,scale=\sca, trimpr]{figs/problem_setting.pdf}}
\par
\subfloat[\label{ours}]{\includegraphics[page=4,scale=\sca, trimp]{figs/problem_setting.pdf}}
\subfloat[\label{smil}]{\includegraphics[page=3, scale=\sca, trimpr]{figs/problem_setting.pdf}}
\caption{
Input patterns of different modality inference problems. Here visual modality is the victim modality that may be missing randomly. (a) modalities are both complete in train and test set; 
(b) modalities are both complete in the train set but the victim modality is completely missing in the test set; 
(c) victim modality is missing randomly in the train set but completely missing in the test set;
(d) modalities are missing with the same probability in train and test set.
}
\end{figure}

In other words, data from each modality are more probable to be \textit{missing at random} (Fig.\ref{ours}~and~\ref{smil}) than \textit{completely present or missing} 
(Fig.\ref{complete}~and~\ref{completelymissing})~\citep{pham2019found,tang-etal-2021-ctfn, zhao-etal-2021-missing}.
Based on the \textit{complete input modality} setting, a family of popular routines regarding the missing-modality inference is to design intricate generative modules attached to the main network and train the model under full supervision with complete modality data. 
By minimizing a customized reconstruction loss, the data restoration (a.k.a. missing data imputation~\cite{van2018flexible}) capability of the generative modules is enhanced~\citep{pham2019found,wang2020transmodality, tang-etal-2021-ctfn} so that the model can be tested in the missing situations (Fig.~\ref{completelymissing}).
However, we notice that (i) if modality-complete data in the training set is scarce, a severe overfitting issue may occur, especially when the generative model is large~\cite{robb2020few,schick2021few,ojha2021few};
(ii) global attention-based (i.e., attention over the whole sequence) imputation may bring unexpected noise since true correspondence mainly exists between temporally adjacent parallel signals~\cite{sakoe1978dynamic}.
\citet{ma2021smil} proposed to leverage unit-length sequential representation to represent the missing modality from the seen complete modality from the input for training.
Nevertheless, such kinds of methods inevitably overlook the \textit{temporal correlation between modality sequences} and only acquire fair performance on the downstream tasks. 

To mitigate these issues, in this paper we present MM-Align, a novel framework for fast and effective multimodal learning on randomly missing multimodal sequences. 
The core idea behind the framework is to \textit{imitate} some indirect but informative clues for the paired modality sequences instead of learning to restore the missing modality directly. 
The framework consists of three essential functional units: 
1) a backbone network that handles the main task;
2) an alignment matrix solver based on the optimal transport algorithm to produce \textit{context-window} style solutions only part of whose values are non-zero and an associated meta-learner to imitate the dynamics and perform imputation in the modality-invariant hidden spaces; 
3) a denoising training algorithm that optimizes and coalesces the backbone network and the learner so that they can work robustly on the main task in missing-modality scenarios.
To empirically study the advantages of our models over current imputation approaches, we test on two settings of the random missing conditions, as shown in Fig.~\ref{ours}~and~Fig.~\ref{smil}, for all possible modality pair combinations. 
To the best of our knowledge, it is the first work that applies optimal transport and denoising training to the problem of inference on missing modality sequences.
In a nutshell, the contribution of this work is threefold:
\begin{itemize}[leftmargin=*]
    \item We propose a novel framework to facilitate the missing modality sequence inference task,
    where we devise an alignment dynamics learning module based on the theory of optimal transport and a denoising training algorithm to coalesce it into the main network.
    \item We design a loss function that enables a context-window style solution for the dynamics solver.
    \item We conduct comprehensive experiments on three publicly available datasets from two multimodal tasks. Results and analysis show that our method leads to a faster and more accurate inference of missing modalities.
\end{itemize}
\section{Related Work}
\subsection{Multimodal Learning}
Multimodal learning has raised prevalent concentration as it offers a more comprehensive view of the world for the task that researchers intend to model~\cite{atrey2010multimodal,lahat2015multimodal,sharma2020multimodal}.
The most fundamental technique in multimodal learning is multimodal fusion~\cite{atrey2010multimodal}, which attempts to extract and integrate task-related information from the input modalities into a condensed representative feature vector. 
Conventional multimodal fusion methods encompass cross-modality attention~\cite{tsai2018learning,tsai2019multimodal,han2021bi}, matrix algebra based method~\cite{zadeh2017tensor,liu2018efficient,liang2019learning} and invariant space regularization~\cite{colombo-etal-2021-improving,han2021improving}.
While most of these methods focus on complete modality input, many take into account the missing modality inference situations~\cite{pham2019found,wang2020transmodality,ma2021smil} as well, which usually incorporate a generative network to impute the missing representations by minimizing the reconstruction loss.
However, the formulation under missing patterns remains underexplored, and
that is what we dedicate to handling in this paper.

\subsection{Meta Learning}
Meta-learning, or learning to learn, is a hot research topic that focuses on how to generalize the learning approach from a limited number of visible tasks to broader task types.
Early efforts to tackle this problem are based on comparison, such as relation networks~\cite{sung2018learning} and prototype-based methods~\cite{snell2017prototypical,qi2018low,lifchitz2019dense}.
Other achievements reformulate this problem as transfer learning~\cite{sun2019meta} and multi-task learning~\cite{pentina2015curriculum,tian2020rethinking}, which devote to seeking an effective transformation from previous knowledge that can be adapted to new unseen data, and further fine-tune the model on the handcrafted hard tasks.
In our framework, we treat the alignment matrices as the training target for the meta-learner.
Combined with a self-adaptive denoising training algorithm, the meta-learner can significantly enhance the predictions' accuracy in the missing modality inference problem. 

\section{Method}
\subsection{Problem Definition}
Given a multimodal dataset $\mathcal{D}=\{\mathcal{D}^{train},\mathcal{D}^{val}, \mathcal{D}^{test}\}$, where $\mathcal{D}^{train},\mathcal{D}^{val}, \mathcal{D}^{test}$ are the training, validation and test set, respectively.
In the training set $\mathcal{D}^{train}=\{(x_i^{m_1},x_i^{m_2},y_i)_{i=1}^n\}$, where $x_i^{m_k} = \{x_{1,1}^{m_k},...,x_{i,t}^{m_k}\}$ are input modality sequences and $m_1,m_2$ denote the two modality types, some modality inputs are missing with probability $p'$. 
Following \citet{ma2021smil}, we assume that modality $m_1$ is complete and the random missing only happens on modality $m_2$, which we call the \textit{victim modality}.
Consequently, we can divide the training set into the complete and missing splits, denoted as $\mathcal{D}^{train}_{c}=\{(x_i^{m_1},x_i^{m_2},y_i)_{i=1}^{n_c}\}$ and $\mathcal{D}^{train}_{m}=\{(x_{i}^{m1},y_i)_{i=n_c+1}^{n}\}$, where $\vert \mathcal{D}^{train}_m\vert/ \vert \mathcal{D}^{train}\vert=p'$.
For the validation and test set, we consider two settings:
a) the victim modality is missing \textit{completely} (Fig.~\ref{ours}), denoted as ``setting A'' in the experiment section;
b) the victim modality is missing with the same probability $p'$ (Fig.~\ref{smil}), denoted as ``Setting B'', in line with~\citet{ma2021smil}.
We consider two multimodal tasks: sentiment analysis and emotion recognition, in which the label $y_i$ represents the sentiment value (polarity as positive/negative and value as strength) and emotion category, respectively.
\begin{figure}[t]
    \centering
    \includegraphics[width=\linewidth,trim=1.6cm 0 1.6cm 0]{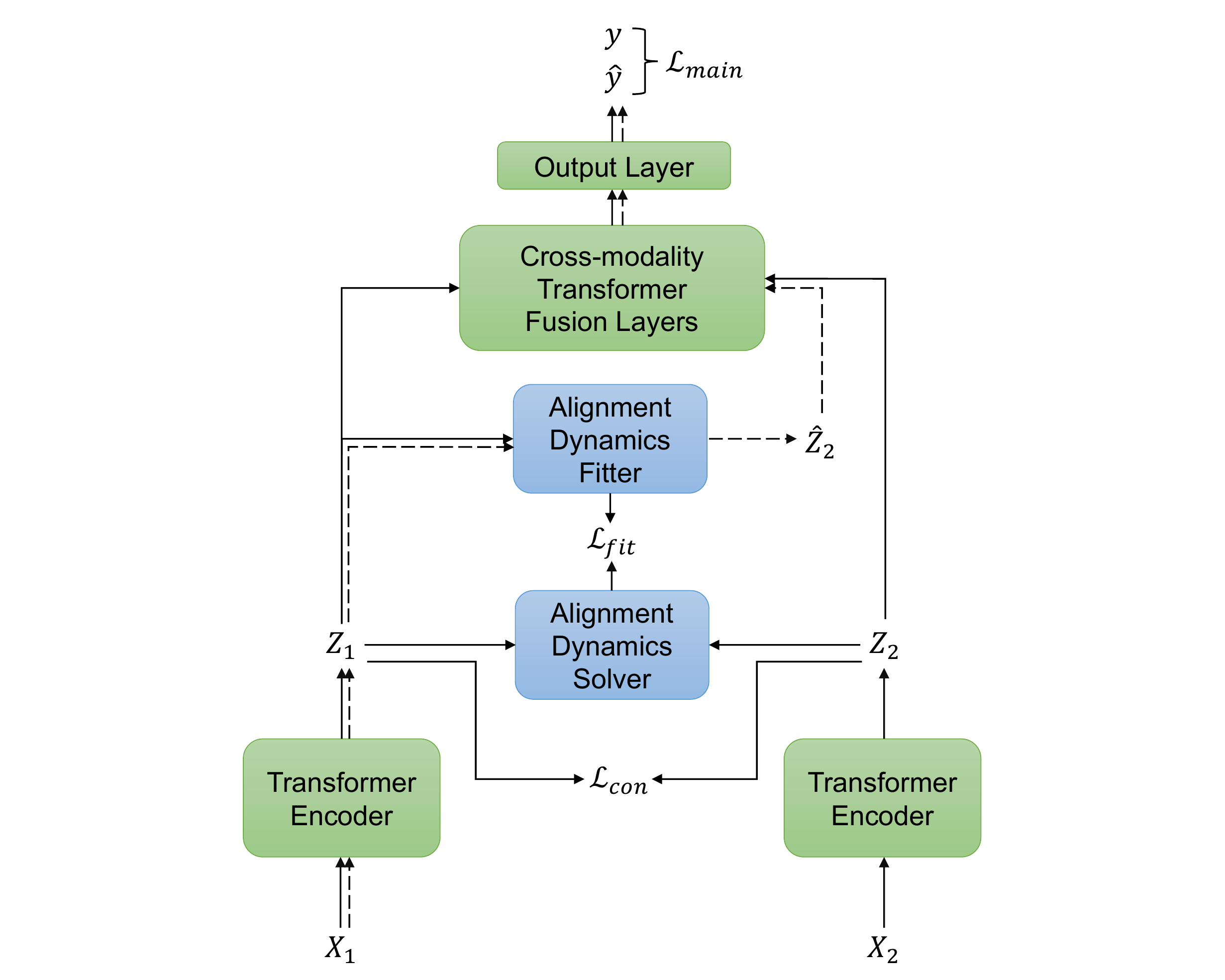}
    \caption{Overall architecture of our framework. Solid lines are the forward paths when training on the modality-complete split and dashed lines are the forward paths when training and testing on the split with missing modality.}
    \label{fig:model_architecture}
\end{figure}

\subsection{Overview}
Our framework encompasses a backbone network (green), an alignment dynamics learner (ADL, blue), and a denoising training algorithm to optimize both the learner and backbone network concurrently.
We highlight the ADL which serves as the core functional unit in the framework.
Motivated by the idea of meta-learning, we seek to generate substitution representations for the missing modality through an indirect imputation clue, i.e., alignment matrices, instead of learning to restore the missing modality by minimizing the reconstruction losses.
To this end, the ADL incorporates an alignment matrix solver based on the theory of optimal transport~\cite{villani2009optimal}, a non-parametric method to capture alignment dynamics between time series~\cite{peyre2019computational,chi2021improving}, as well as an auxiliary neural network to fit and generate meaningful representations as illustrated in~\cref{sec:ADL}.

\subsection{Architecture}
\paragraph{Backbone Network}
The overall architecture of our framework is depicted in Fig.~\ref{fig:model_architecture}.
We harness MulT~\cite{tsai2019multimodal}, a fusion network derived from Transformer~\cite{vaswani2017attention} as the backbone structure since we find a number of its variants in preceding works acquire promising outcomes in multimodal~\cite{wang2020transmodality,han2021bi,tang-etal-2021-ctfn}. 
MulT has two essential components: the unimodal self-attention encoder and bimodal cross-attention encoder. 
Given modality sequences $x^{m_1}, x^{m_2}$ (for unimodal self-attention we have $m_1 = m_2$) as model's inputs, after padding a special token $x_0^{m_1}=x_0^{m_2} = $\texttt{[CLS]} to their individual heads, a single transformer layer~\cite{vaswani2017attention} encodes a sequence through a multi-head attention (MATT) and feed-forward network (FFN) as follows:
\begin{align}
    Q = x^{m_1} & W_Q, K = x^{m_2}W_K, V=x^{m_2}W_V 
    \label{eq:1} \\
    \hat{Z}^{21} &= \mathrm{MATT}(Q,K,V)+x^{m_1} \\
    Z^{21} &= \mathrm{FFN}(\hat{Z}^{21})+ \mathrm{LN}(\hat{Z}^{21}) \label{eq:3}
\end{align}
where LN is layer normalization. 
In our experiments, we leverage this backbone structure for both input modality encoding and multimodal fusion.

\paragraph{Output Layer} We extract the head embeddings $z^{12}_0,z^{21}_0$ from the output of the fusion network as features for regression. 
The regression network is a two-layer feed-forward network: 
\begin{equation}
    \hat{y} = W_2(\textrm{tanh}(W_1[z_0^{12}, z_0^{21}]+b_1)+b_2
    \label{eq:4}
\end{equation}
where $[\boldsymbol{\cdot},\boldsymbol{\cdot},\cdots]$ is the concatenation operation.
The mean squared error (MSE) is adopted as the loss function for the regression task:
\begin{equation}
    \mathcal{L}_{main} = \mathrm{MSE}(\hat{y}, y)
    \label{eq:5}
\end{equation}

\subsection{Alignment Dynamics Learner (ADL)}
\label{sec:ADL}
The learner has two functional modules, named as \emph{alignment dynamics solver} and \emph{fitter}, as shown in Fig.~\ref{fig:model_architecture}. 
It also runs in two functional modes, namely learning and decoding.
ADL works in learning mode when the model is trained on the complete data (marked by the solid lines in Fig.~\ref{fig:model_architecture}). 
The decoding mode is triggered when one of the modalities is missing, which happens in the training time on the missing splits and the entire test time (marked by the dashed lines in Fig.~\ref{fig:model_architecture}).
\paragraph{Learning Mode} In the learning mode, the solver calculates an alignment matrix which provides the information about temporal correlations between the two modality sequences. 
Similar to the previous works~\cite{peyre2019computational,chi2021improving}, this problem can be formulated as an optimal transport (OT) task:
\begin{equation}
    \min_A \sum_{i,j} A_{ij}M_{ij}
\label{eq:goal}
\end{equation}
where $A$ is the transportation plan that implies the alignment information~\cite{peyre2019computational} and $M$ is the cost matrix. 
The subscript $ij$ represents the component from the $i$th timestamp in the source modality to the $j$ th timestamp in the target modality.
Different from~\citet{peyre2019computational} and~\citet{chi2021improving} which allow alignment between any two positions of the two sequences, we believe that in parallel time series, the temporal correlation mainly exists between signals inside a time-specific ``window'' (i.e., $\vert j-i \vert \leq W$, where $W$ is the window size)~\cite{sakoe1978dynamic}.
Additionally, the cost function should be negatively correlated to the similarity (distance), as one of the problem settings in the original OT problem.
To realize these basic motivations, we borrowed the concept of barrier function~\cite{nesterov2018lectures} and define the cost function for our optimal transport problem as:
\begin{equation}
    M_{ij} = 
    \begin{cases}
        1 - \cos{(z_i^1, z_j^2)}, & \vert i - j \vert \leq K \\
        \infty, & \vert i - j \vert > K
    \end{cases}
    \label{eq:ot_cost_func}
\end{equation}
where $z_i^m$ is the representation of modality $m$ at timestamp $i$ and $cos(\cdot,\cdot)$ is the cosine value of two vectors.
We will show that such a type of transportation cost function ensures a context-window style alignment solution and also provide a proof in~\cref{sec:sol_math}.
To solve Eq. \eqref{eq:goal}, a common practice is to add an entropic regularization term:
\begin{equation}
    \min_A \sum_{i,j} A_{ij}M_{ij}-\mu A_{ij}\log A_{ij}
\label{eq:entropy_form}
\end{equation}
The unique solution $A^*$ can be calculated through Sinkhorn's algorithm~\cite{peyre2019computational}:
\begin{equation}
    A^* = \mathrm{diag}(\mathbf{u})K\mathrm{diag}(\mathbf{v}),
    K = \exp{(M/\mu)}
\end{equation}
The vector $\mathbf{u}$ and $\mathbf{v}$ are obtained through the following iteration until convergence:
\begin{align}
    \mathbf{v}^{t=0} &= \mathbf{1}_m 
    \label{eq:skh_step1} \\
    \mathbf{u}^{t+1} &=\frac{\mathbf{1}_n}{K\mathbf{v}^t},\quad
    \mathbf{v}^{t+1} =\frac{\mathbf{1}_n}{K\mathbf{u}^{t+1}}
    \label{eq:skh_step2}
\end{align}
After quantifying the temporal correlation into alignment matrices, we enforce the learner to fit those matrices so that it can automatically approximate the matrices from the non-victim modality in the decoding mode.
Specifically, a prediction network composed of a gated recurrent unit~\cite{chung2014empirical} and a linear projection layer takes the shared representations of the complete modality as input and outputs the prediction value for entries:
\begin{align}
    \hat{T}=\mathrm{softmax}(\mathrm{Linear}(\mathrm{GRU}(Z^1;\psi_r); \psi_p))
\end{align}
where $\psi=\{\psi_r,\psi_t\}$ is the collection of parameters in the prediction network. $\hat{T} = \{\hat{t}_1, \hat{t}_2, ..., \hat{t}_l\} \in \mathbb{R}^{l\times(2W+1)}$ are the predictions for $A^*$ and $\hat{t}_i \in \mathbb{R}^{2W+1}$ is the prediction for the alignment matrix segment $A^*_{i,i-W:i+W}$, i.e., the alignment components which span within the radius of $W$ centered at current timestamp $i$.
We reckon the mean squared error (MSE) between ``truths" generated from the solver and predictions to calculate the fitting loss:
\begin{equation}
    \mathcal{L}_{fit} = \frac{1}{(2W+1)l}\sqrt{\sum_i\sum_{j=i-W}^{i+W} (A^*_{ij}-\hat{T}_{ij})^2}
    \label{eq:fit_loss}
\end{equation}
where the summation is over the entries within context windows and we define $A^*_{ij}=0$ if $j\leq 0$ or $j>l$ for better readability.

\paragraph{Decoding Mode} 
In this mode, the learner behaves like a decoder that strives to generate meaningful substitution to the missing modality sequences.
The learner first decodes an alignment matrix $\hat{A}$ via the fitting network whose parameters are frozen during this stage.
Afterward, the imputation of the missing modality at position $j$ can be obtained through the linear combination of alignment matrices and visible sequences:
\begin{equation}
    \hat{z}_j^2 = \sum_{i=j-W}^{j+W}\hat{A}_{ij}z_i^1
    \label{eq:imp}
\end{equation}
We concatenate all these vectors to construct the imputation for the missing modality $\hat{Z}^2$ in the shared space:
\begin{equation}
    \hat{Z^2} = [\hat{z}_0^2, \hat{z}_1^2, \hat{z}_2^2, ..., \hat{z}_l^2]
    \label{eq:imp_sq}
\end{equation}
where $\hat{z}_0^2$ is reassigned by the initial embedding of the \texttt{[CLS]} token. 
The imputation results together with the complete modality sequences are then fed into the fusion network (Eq.~\eqref{eq:1} \textasciitilde \eqref{eq:3}) to continue the subsequent procedure.
\begin{algorithm}[ht!]
\small
\SetAlgoLined
\KwIn{$\mathcal{D}^{train} = \{\mathcal{D}^{train}_c, \mathcal{D}^{train}_m \}$, learning rate $\eta_{fit}, \eta_{main}$, parameters of the backbone network $\theta=\{\theta_{enc},\theta_{fu},\theta_{out}\}$ and the alignment dynamics learner $\psi=\{\psi_d\}$, batch size $n_b$, $\lambda$
}
\tcp{Warm-up Stage}
\For{\upshape each warm-up epoch}{
    \For{\upshape each $\mathcal{B}=\{\cup_{i=1}^{n_b}(x^{m_1}_i,x^{m_2}_i, y_i)\} \subset \mathcal{D}^{train}_c$}{
        \upshape Compute $\mathcal{L}_{main}, \mathcal{L}_{con}$ by Eq.~\eqref{eq:1}\textasciitilde\eqref{eq:5},~\eqref{eq:15}, \eqref{eq:16} \\
        $\theta \leftarrow \theta - \eta_{main}\nabla_{\theta}(\mathcal{L}_{main} + \lambda\mathcal{L}_{cons})$
    }
}
\For{\upshape each training epoch}{
\tcp{Train on the complete split} 
\For{\upshape each $\mathcal{B}=\{\cup_{i=1}^{n_b}(x^{m_1}_i,x^{m_2}_i, y_i)\} \subset \mathcal{D}^{train}_c$}{
    \upshape Compute $A^*$ by Sinkhorn algorithm according to Eq.~\eqref{eq:ot_cost_func}\textasciitilde\eqref{eq:skh_step2} \\
    \upshape Compute $\mathcal{L}_{fit}$ according to~\eqref{eq:fit_loss}; \\
    \tcp{Tune the dynamics learner}
    $\psi \leftarrow \psi - \eta_{fit}\nabla_{\psi}\mathcal{L}_{fit}$ \\
    \upshape Compute $\mathcal{L}_{main}, \mathcal{L}_{con}$ according to Eq.~\eqref{eq:1}\textasciitilde\eqref{eq:5},~\eqref{eq:15}, \eqref{eq:16} \\
    \tcp{Tune the backbone network}
    $\theta \leftarrow \theta - \eta_{main}\nabla_{\theta}(\mathcal{L}_{main} + \lambda\mathcal{L}_{cons})$
 }
\tcp{Train on the missing split}
\For{\upshape each $\mathcal{B}=\{\cup_{i=1}^{n_b}(x^{m_1}_i, y_i)\} \subset \mathcal{D}^{train}_m$}{
    \upshape Impute the representation sequences of the missing modality $\hat{Z}_i^2$ by Eq.~\eqref{eq:imp}~\eqref{eq:imp_sq} and then $\mathcal{L}_{main}$ by Eq.~\eqref{eq:1}\textasciitilde\eqref{eq:5},~\eqref{eq:15}, \eqref{eq:16}
     $\theta \leftarrow \theta - \eta_{main}\nabla_{\theta}\mathcal{L}_{main}$
 }
}
\caption{Denoising Training}
\label{alg:1}
\end{algorithm}

\subsection{Denoising Training}
Inspired by previous work in data imputation~\cite{kyono2021miracle}, we design a denoising training algorithm to promote prediction accuracy and imputation quality concurrently, as shown in Alg.~\ref{alg:1}.
In the beginning, we warm up the model on the complete split of the training set.
We utilize two transformer encoders to project input modality sequences $x^{m_1}$ and $x^{m_2}$ into a shared feature space, denoted as $Z^1$ and $Z^2$. 
Following~\citet{han2021improving}, we apply a contrastive loss~\cite{chen2020simple}~as the regularization term to force a similar distribution of the generated vectors $Z^1$ and $Z^2$:
\begin{equation}
    \mathcal{L}_{con} = -\frac{1}{N_b}\sum_i \log\frac{\phi(Z^1_i,Z_i^2)}{\sum_{j} \phi(Z_i^{1},Z_j^{2})}
    \label{eq:15}
\end{equation}
where the summation is over the whole batch of size $N_b$ and $\phi$ is a score function with an annealing temperature $\tau$ as the hyperparameter:
\begin{equation}
    \phi(s,t) = \exp{(s^Tt/\tau)}
    \label{eq:16}
\end{equation}
Next, the denoising training loop proceeds to couple the ADL and backbone network. 
In a single loop, we first train the alignment dynamics learner (line 9\textasciitilde11), then we train the backbone network on the complete split (line 12\textasciitilde13) and missing split (line 15\textasciitilde17).
Since the learner training process uses the modality-complete split, and we found in experiments (\cref{sec:abl}) that model's performance stays nearly constant if the tuning for the learner and the main network occurs concurrently on every batch,
we merge them into a single loop (line 8\textasciitilde14) to reduce the redundant batch iteration. 

\section{Experiments}
\subsection{Datasets}
We utilize CMU-MOSI~\cite{zadeh2016multimodal} and CMU-MOSEI~\cite{zadeh2018multimodal} for sentiment prediction, and MELD~\cite{poria2019meld} for emotion recognition, to create our evaluation benchmarks.
The statistics of these datasets and preprocessing steps can be found in~\cref{sec:ds_stat}.
All these datasets consist of three parallel modality sequences---text (t), visual (v) and acoustic (a). 
In a single run, we extract a pair of modalities and select one of them as the victim modality which we then randomly remove $p'= 1-p$ of all its sequences.
Here $p$ is the surviving rate for the convenience of description.
We preprocess test sets as Fig.~\ref{ours}~(remove all victim modality samples) in setting A and Fig.~\ref{smil}~(randomly remove $p'$ of victim modality samples) in setting B.
Setting B inherits from~\citet{ma2021smil}~while the newly added setting A is considered as a complementary test case of more severe missing situations, which can compare the efficacy of pure imputation methods and enrich the connotation of robust inference.   
We run experiments with two randomly picking $p\in\{10\%,50\%\}$--- 
dissimilar to~\citet{ma2021smil}, we enlarge the gap between two $p$ values to strengthen the distinction between these settings. 

\subsection{Baselines and Evaluation Metrics}
We compare our models with the following relevant and strong baselines:
\begin{itemize}[leftmargin=*]
\item \textbf{Supervised-Single} trains and tests the backbone network on a single complete modality, which can be regarded as the~\textit{lower bound} (LB) for all the baselines.
\item \textbf{Supervised-Double} trains and tests the backbone network on a pair of complete modalities, which can be regarded as the~\textit{upper bound} (UB).
\item \textbf{MFM}~\cite{tsai2018learning} learns modality-specific generative factors that can be produced from other modalities at training time and imputes the missing modality based on these factors at test time.
\item \textbf{SMIL}~\cite{ma2021smil} imputes the sequential representation of the missing modality by linearly adding clustered center vectors with weights from learned Gaussian distribution.
\item \textbf{Modal-Trans}~\cite{wang2020transmodality,tang-etal-2021-ctfn} builds a cyclic sequence-to-sequence model and learns bidirectional reconstruction.
\end{itemize}

The characteristics of all these models are listed for comparison in Table~\ref{tab:model_comp}. 
Previous work relies on either a Gaussian generative or sequence-to-sequence formulation to reconstruct the victim modality or its sequential representations, while our model adopts none of these architectures. 
We run our models under 5 different splits and report the average performance. 
The training details can be found in~\cref{sec:hp_search}.

We compare these models on the following metrics: for the sentiment prediction task, we employ the mean absolute error (MAE) which quantifies how far the prediction value deviates from the ground truth, and the binary classification accuracy (Acc-2) that counts the proportion of samples correctly classified into positive/negative categories; for emotion recognition task we compare the average F1 score over seven emotional classes.
\begin{table}[ht!]
    \centering
    \small
    \begin{tabular}{c|ccc}
    \toprule
         \multirow{2}{*}{Model} & \multirow{2}{*}{\shortstack{Generative \\ Gaussian}} & \multirow{2}{*}{Recon} & \multirow{2}{*}{Seq2Seq} \\
         ~ & & & \\
         \hline
         MFM  & \xmark & \cmark & \cmark \\
         SMIL & \cmark & \cmark & \xmark \\
         Modal-Trans & \xmark & \cmark & \cmark \\
         \modelname~(Ours) & \xmark & \xmark  & \xmark  \\
    \bottomrule
    \end{tabular}
    \caption{Model characteristics.}
    \label{tab:model_comp}
\end{table}

\renewcommand{\arraystretch}{1.005}
\begin{table*}[ht]
    \centering
    \small
    \resizebox{\linewidth}{!}{
    \begin{tabular}{l*{3}{|c c | c c}}
        \toprule
        \multirow{3}{*}{Method} & 
        \multicolumn{4}{c|}{T $\rightarrow$ V} & \multicolumn{4}{c|}{V$\rightarrow$A}  & \multicolumn{4}{c}{A$\rightarrow$T}   \\
        \cline{2-13}
           ~ & \multicolumn{2}{c|}{Setting A} & \multicolumn{2}{c|}{Setting B} & \multicolumn{2}{c|}{Setting A} & \multicolumn{2}{c|}{Setting B} & \multicolumn{2}{c|}{Setting A} & \multicolumn{2}{c}{Setting B} \\
         \cline{2-13}
         ~ & \textbf{MAE$\downarrow$} & \textbf{Acc-2$\uparrow$} & \textbf{MAE$\downarrow$} & \textbf{Acc-2$\uparrow$} & \textbf{MAE$\downarrow$} & \textbf{Acc-2$\uparrow$} & \textbf{MAE$\downarrow$} & \textbf{Acc-2$\uparrow$} & \textbf{MAE$\downarrow$} & \textbf{Acc-2$\uparrow$} & \textbf{MAE$\downarrow$} & \textbf{Acc-2$\uparrow$} \\
        \midrule
        LB & 1.242 & 68.6 & 1.242 & 68.6 & 1.442 & 46.4 & 1.442 & 46.4 & 1.440 & 42.2 & 1.440 & 42.2 \\
        UB & 1.019 & 77.7 & 1.019 & 77.7 & 1.413 & 57.8 & 1.413 & 57.8 & 1.081 & 75.8 & 1.081 & 75.8 \\
        \cline{1-13}
        MFM & 1.103 & 71.0 & 1.093 & 73.2 & 1.456 & 43.5 & 1.452 & 43.9  & 1.477 & 42.2 & 1.454 & 42.2  \\
        SMIL & 1.073 & 74.2 & 1.052 & 75.3 & 1.442 & 45.9 & 1.438 & 46.5 & 1.447 & 43.3 & 1.439 & 45.4  \\
        Modal-Trans & 1.052 & 75.5 & 1.041 & 75.8 & 1.428 & 49.4 & 1.425 & 49.7 & 1.435 & 48.7 & 1.432 & 48.9  \\
        \cline{1-13}
        MM-Align & \textbf{1.028}$^\natural$ & \textbf{76.9}$^\natural$ & \textbf{1.027} & \textbf{77.0} & \textbf{1.416}$^\natural$ & \textbf{52.0}$^\natural$ & \textbf{1.411}$^\natural$ & \textbf{53.1}$^\natural$  & \textbf{1.426} & \textbf{51.5}$^\natural$ & \textbf{1.414}$^\natural$  & \textbf{52.0}$^\natural$  \\
        \midrule
        \multirow{3}{*}{} & 
        \multicolumn{4}{c|}{V $\rightarrow$ T} & \multicolumn{4}{c|}{A$\rightarrow$V}  & \multicolumn{4}{c}{T$\rightarrow$A}   \\
        \cline{2-13}
           ~ & \multicolumn{2}{c|}{Setting A} & \multicolumn{2}{c|}{Setting B} & \multicolumn{2}{c|}{Setting A} & \multicolumn{2}{c|}{Setting B} & \multicolumn{2}{c|}{Setting A} & \multicolumn{2}{c}{Setting B} \\
         \cline{2-13}
         ~ & \textbf{MAE$\downarrow$} & \textbf{Acc-2$\uparrow$} & \textbf{MAE$\downarrow$} & \textbf{Acc-2$\uparrow$} & \textbf{MAE$\downarrow$} & \textbf{Acc-2$\uparrow$} & \textbf{MAE$\downarrow$} & \textbf{Acc-2$\uparrow$} & \textbf{MAE$\downarrow$} & \textbf{Acc-2$\uparrow$} & \textbf{MAE$\downarrow$} & \textbf{Acc-2$\uparrow$} \\
        \midrule
        LB & 1.442 & 46.3 & 1.442 & 46.3 & 1.440 & 42.2 & 1.440 & 42.2 & 1.242 & 68.6 & 1.242 & 68.6 \\
        UB & 1.019 & 77.7 & 1.019 & 77.7 & 1.413 & 57.8 & 1.413 & 57.8 & 1.081 & 75.8  & 1.081 & 75.8 \\
        \cline{1-13}
        MFM & 1.479 & 42.2 & 1.429 & 51.9 & 1.454 & 42.2 & 1.455 & 42.2 & 1.078 & 72.9 & 1.082 & 73.7  \\
        SMIL & 1.448 & 44.2 & 1.447 & 43.3 & 1.442 & 45.9 &  1.438 & 47.3 & 1.060 & 75.5 & 1.089 & 74.9 \\
        Modal-Trans & 1.429 & 50.3 & 1.420 & 53.1  & 1.439 & 47.4 & 1.442  & 48.3 & 1.052 & 75.2 & 1.073 & 74.3  \\
        \cline{1-13}
        MM-Align & \textbf{1.415}$^\natural$ & \textbf{52.7}$^\natural$ & \textbf{1.410} & \textbf{53.4} & \textbf{1.427}$^\natural$ & \textbf{49.9}$^\natural$ &  \textbf{1.426}$^\natural$  & \textbf{50.7}$^\natural$ & \textbf{1.028}$^\natural$ & \textbf{76.7}$^\natural$ & \textbf{1.032}$^\natural$ & \textbf{76.6}$^\natural$ \\
        \bottomrule
    \end{tabular}
    }
    \caption{Results on the CMU-MOSI dataset~($p=10$). The reported results are the average of five runs using the same set of hyperparameters and different random seeds. ``A $\to$ B'' means the imputation from the complete modality A to the missing modality B at the test time. $\natural$: results of our model are significantly better than the highest baselines with p-value $<$ 0.05 based on the paired t-test.}
    \label{table:mosi}
\end{table*}

\renewcommand{\arraystretch}{1.005}
\begin{table*}[ht]
    \centering
    \small
    \resizebox{\linewidth}{!}{
    \begin{tabular}{l*{3}{|c c | c c}}
        \toprule
        \multirow{3}{*}{Method} & 
        \multicolumn{4}{c|}{T $\rightarrow$ V} & \multicolumn{4}{c|}{V$\rightarrow$A}  & \multicolumn{4}{c}{A$\rightarrow$T}   \\
        \cline{2-13}
           ~ & \multicolumn{2}{c|}{Setting A} & \multicolumn{2}{c|}{Setting B} & \multicolumn{2}{c|}{Setting A} & \multicolumn{2}{c|}{Setting B} & \multicolumn{2}{c|}{Setting A} & \multicolumn{2}{c}{Setting B} \\
         \cline{2-13}
         ~ & \textbf{MAE$\downarrow$} & \textbf{Acc-2$\uparrow$} & \textbf{MAE$\downarrow$} & \textbf{Acc-2$\uparrow$} & \textbf{MAE$\downarrow$} & \textbf{Acc-2$\uparrow$} & \textbf{MAE$\downarrow$} & \textbf{Acc-2$\uparrow$} & \textbf{MAE$\downarrow$} & \textbf{Acc-2$\uparrow$} & \textbf{MAE$\downarrow$} & \textbf{Acc-2$\uparrow$} \\
        \midrule
        LB & 0.687  & 77.4 & 0.687 & 77.4 & 0.836 & 61.3 & 0.836 &  61.3 & 0.851 & 62.9 & 0.851 & 62.9 \\
        UB & 0.615 & 81.3 & 0.615 & 81.3 & 0.707 & 79.5 & 0.707 & 79.5 & 0.613 & 80.9 & 0.613 & 80.9 \\
        \cline{1-13}
        MFM & 0.658 & 79.2 & 0.645 & 80.0 & 0.827 & 61.5 & 0.818 & 61.9 & 0.836 & 64.3 & 0.830 & 63.6 \\
        SMIL & 0.680 & 78.3 & 0.648 & 78.5 & 0.819 & 64.3 & 0.816 & 63.6 & 0.840 & 62.9 & 0.839 & 63.0 \\
        Modal-Trans & 0.645 & 79.6 & 0.647 & 79.6 & 0.818 & 64.7 & 0.815 & 65.4 & 0.827 & 64.9 & 0.823 & 65.6 \\
        \cline{1-13} 
        MM-Align & \textbf{0.637}$^\natural$ & \textbf{80.8}$^\natural$ & \textbf{0.638}$^\natural$ & \textbf{81.1}$^\natural$  & \textbf{0.811}$^{\natural}$ & \textbf{65.9}$^\natural$ & \textbf{0.813} & \textbf{66.2}$^\natural$  & \textbf{0.824} & \textbf{65.3} & \textbf{0.817} & \textbf{66.3} \\
        \midrule
        \multirow{3}{*}{} &  
        \multicolumn{4}{c|}{V $\rightarrow$ T} & \multicolumn{4}{c|}{A$\rightarrow$V}  & \multicolumn{4}{c}{T$\rightarrow$A}   \\
        \cline{2-13}
           ~ & \multicolumn{2}{c|}{Setting A} & \multicolumn{2}{c|}{Setting B} & \multicolumn{2}{c|}{Setting A} & \multicolumn{2}{c|}{Setting B} & \multicolumn{2}{c|}{Setting A} & \multicolumn{2}{c}{Setting B} \\
         \cline{2-13}
         ~ & \textbf{MAE$\downarrow$} & \textbf{Acc-2$\uparrow$} & \textbf{MAE$\downarrow$} & \textbf{Acc-2$\uparrow$} & \textbf{MAE$\downarrow$} & \textbf{Acc-2$\uparrow$} & \textbf{MAE$\downarrow$} & \textbf{Acc-2$\uparrow$} & \textbf{MAE$\downarrow$} & \textbf{Acc-2$\uparrow$} & \textbf{MAE$\downarrow$} & \textbf{Acc-2$\uparrow$} \\
        \midrule
        LB & 0.836 & 61.3 & 0.836 & 61.3 & 0.851 & 62.9 & 0.851 & 62.9 & 0.687 & 77.4 & 0.687 & 77.4 \\
        UB & 0.615 & 81.3 & 0.615 & 81.3 & 0.707 & 79.5 & 0.707 & 79.5 & 0.613 & 80.9 & 0.613 & 80.9 \\
        \cline{1-13}
        MFM & 0.821 & 62.0 & 0.817 & 61.7 & 0.842 & 62.7 & 0.828 & 63.9 & 0.658 & 79.1 & 0.645 & 79.7 \\
        SMIL & 0.820 & 63.1 & 0.816 & 63.5 & 0.838 & 63.2 & 0.842 & 62.4 & 0.684 & 78.5 & 0.684 & 77.4  \\
        Modal-Trans & 0.817 & 65.1 & 0.814 & 65.7 & 0.832 & 64.6 & 0.823 & 65.1 & 0.643 & 79.9 & 0.645 & 79.4 \\
        \cline{1-13}
        MM-Align & \textbf{0.811}$^\natural$ & \textbf{66.2}$^\natural$  & \textbf{0.806}$^\natural$  & \textbf{66.9}$^\natural$  & \textbf{0.822}$^\natural$  & \textbf{65.4}$^\natural$  & \textbf{0.818} & \textbf{65.7} & \textbf{0.635}$^\natural$  & \textbf{81.0}$^\natural$  & \textbf{0.637}$^\natural$  & \textbf{80.9}$^\natural$  \\
        \bottomrule
        \end{tabular}
    }
    \caption{Results on the CMU-MOSEI dataset ($p=10)$. Notations share the same meaning as the last table.}
    \label{table:mosei}
\end{table*}

\renewcommand{\arraystretch}{1.02}
\begin{table}[ht]
    \centering
    \small
    \resizebox{\linewidth}{!}{
    \begin{tabular}{l*{3}{|c|c}}
        \toprule
        \multirow{2}{*}{Method} &
         A & B &  A &  B & A & B \\
        \cline{2-7}
        ~ & \multicolumn{2}{c|}{T $\rightarrow$ V} & \multicolumn{2}{c|}{V$\rightarrow$A}  & \multicolumn{2}{c}{A$\rightarrow$T}   \\
        \midrule
        LB & 54.0 & 54.0 & 31.3 & 31.3 & 31.3 & 31.3  \\
        UB  & 55.8 & 55.8 & 32.1 & 32.1 & 55.9 & 55.9 \\
        \cline{1-7}
        MFM    & 54.0 & 53.9 & 31.3 & 31.3 & 31.3 & 43.1 \\
        SMIL   & 54.4 & 54.2 & 31.3 & 31.3 & 31.3 & 43.5\\
        Modal-Trans & 55.0 & 54.8 & 31.3 & 31.4 & 31.5 & 44.4 \\
        \cline{1-7}
        MM-Align & \textbf{55.7} & \textbf{55.7} & \textbf{31.9} & \textbf{31.9}$^{\natural}$ & 31.5 & \textbf{45.5}  \\
        \cline{1-7}
        ~ & \multicolumn{2}{c|}{V $\rightarrow$ T} & \multicolumn{2}{c|}{A $\rightarrow$V}  & \multicolumn{2}{c}{T $\rightarrow$A}   \\
        \midrule
        LB & 31.3 & 31.3 & 31.3 & 31.3 & 54.0 & 54.0 \\
        UB  & 55.8 & 55.8 & 32.1 & 32.1 & 55.9 & 55.9  \\
        \cline{1-7}
        MFM & 31.4 & 43.6 & 31.3 & 31.3 & 54.2 & 54.1   \\
        SMIL & 31.4 & 43.9 & 31.3 & 31.3 & 54.5 & 54.2   \\
        Modal-Trans & 31.6 & 44.2 & 31.3 & 31.3 & 55.0 & 54.8  \\
        \cline{1-7}
        MM-Align & \textbf{32.3} & \textbf{45.4} & 31.3 & \textbf{32.0} & \textbf{55.6} & \textbf{55.7}  \\
        \bottomrule
    \end{tabular}
    }
    \caption{Results on MELD~($p=50\%$). Notations share the same meaning as the last table.}
    \label{table:meld}
\end{table}

\subsection{Results}
Due to the particularities of three datasets, We report the results of the smallest $p$ values when most of these baselines yield $1\%$ higher results than the lower bound in Table~\ref{table:mosi},~\ref{table:mosei}~and~\ref{table:meld}.
From them we mainly have the following observations:

First, Compared with lower bounds, in setting A where models are tested with only the non-victim modality, our method gains 6.6\%\textasciitilde9.3\%, 2.4\%\textasciitilde4.9\% accuracy increment on the CMU-MOSI and CMU-MOSEI dataset and 0.6\%\textasciitilde1.7\% F1 increment on the MELD dataset (except A$\to$V and A$\to$T).
Besides,~\modelname~significantly outperforms all the baselines in most settings.
These facts indicate that leveraging the local alignment information as indirect clues facilitates to performing robust inference on missing modalities.

Second, model performance varies greatly especially when the non-victim modality alters.
It has been pointed out that three modalities do not play an equal role in multimodal tasks~\cite{tsai2019multimodal}. 
Among them, the text is usually the predominant modality that contributes majorly to accuracy, while visual and acoustic have weaker effects on the model's performance.
From the results, it is apparent that if the source modality is predominant, the model's performance gets closer to or even surpasses the upper bound, which reveals that the predominant modality can also offer richer clues to facilitate the dynamics learning process than other modalities.

Third, when moving from setting A to setting B by adding parallel sequences of the non-victim modality in the test set, results incline to be constant in most settings. 
Intuitively, performance should become better if more parallel data are provided.
However, as most of these models are unified and must learn to couple the restoration/imputation module and backbone network, the classifier inevitably falls into the dilemma that it should adapt more to the true parallel sequences or the mixed sequences since both are included patterns in a training epoch.
Hence sometimes setting B would not perform evidently better than setting A.
Particularly, we find that when Modal-Trans encounters overfitting,~\modelname~can alleviate this trend, such as T$\to$A in all three datasets.

Additionally,~\modelname~acquires a 3\textasciitilde4$\times$ speed-up in training. We record the time consumption and provide a detailed analysis in~\cref{sec:comp_analysis}~and~\ref{sec:inf_speed}.

\begin{table}[ht]
    \centering
    \small
    \resizebox{\linewidth}{!}{
    \begin{tabular}{l*{3}{|c c}}
        \toprule
        \multirow{2}{*}{Settings} & 
        \multicolumn{2}{c|}{T $\rightarrow$ V} & \multicolumn{2}{c|}{V$\rightarrow$A} & \multicolumn{2}{c}{A$\rightarrow$T} \\
         ~ & \textbf{MAE$\downarrow$} & \textbf{Acc-2$\uparrow$} & \textbf{MAE$\downarrow$} & \textbf{Acc-2$\uparrow$} & \textbf{MAE$\downarrow$} & \textbf{Acc-2$\uparrow$} \\
        \cline{1-7}
        \modelname & \textbf{1.028} & \textbf{76.9} & \textbf{1.416} & \textbf{52.0} & 1.426 & 51.5  \\
        w/o $\mathcal{L}_{con}$ & 1.037 & 76.7 & 1.422 & 51.8 & 1.432 & 49.5 \\
        w/o $\mathcal{L}_{fit}$ & 1.085 & 72.2 & 1.437 & 47.3 & 1.448 & 44.6 \\
        w/o SI & 1.033 & 76.6 & 1.425 & 51.9 & \textbf{1.419} & \textbf{51.8} \\
        \bottomrule
    \end{tabular}
    }
    \caption{Results of ablation experiments on CMU-MOSI dataset.}
    \label{tab:abl_study}
\end{table}

\begin{figure*}[ht]
    \centering
    \includegraphics[width=\textwidth, trim=0 5cm 0 6cm]{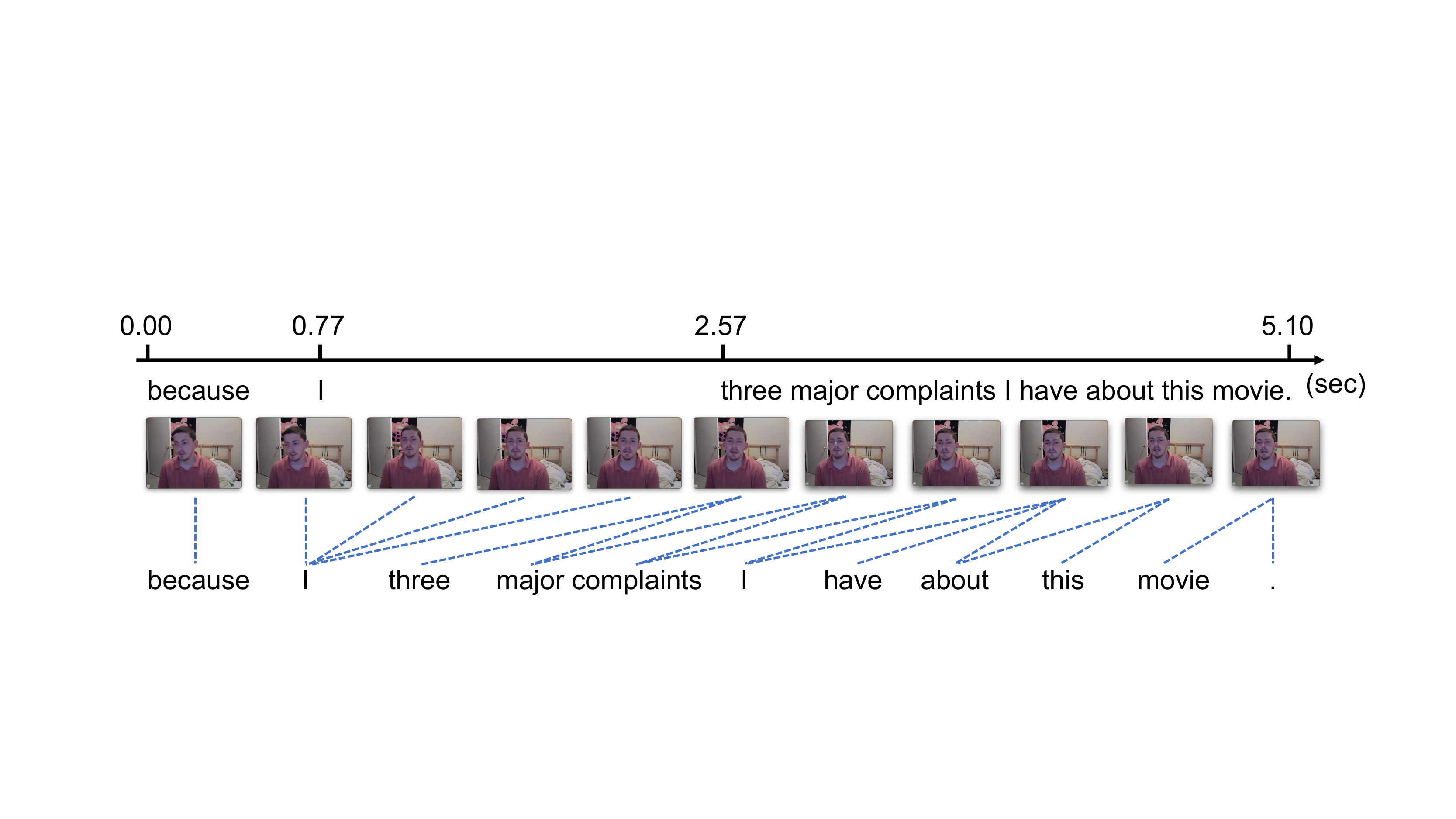}
    \caption{An example from CMU-MOSI dataset. The text below the time axis is aligned to the starting time of its pronunciation. The pictures are the central frame of each cluster that lasts the same time interval. The dashed lines connect each word with the frames of its appearance in the video.}
    \label{fig:case_study}
\end{figure*}

\subsection{Ablation Study}
\label{sec:abl}
We run our model under the following ablative settings on three randomly chosen modality pairs from the CMU-MOSI dataset in setting A: 1) removing the contrastive loss which serves as the invariant space regularizer; 
2) removing the fitting loss so that the ADL only generates a random alignment matrix when running in the inference mode; 
3) separating the single iteration (SI) over the complete split that concurrently optimizes the fitter and backbone network in Alg.~\ref{alg:1} into two independent loops.
The results of these experiments are displayed in Table~\ref{tab:abl_study}.
We witness a performance drop after removing the contrastive loss, and the drop is higher if we disable the ADL, which implies the benefits from the alignment dynamics-based generalization process on the modality-invariant hidden space.
Finally, merging two optimization steps will not cause performance degradation. 
Therefore it is more time-efficient to design the denoising loop as Alg.~\ref{alg:1} to prevent an extra dataset iteration.

\section{Analysis}
\paragraph{Impact of the Window Size} 
To further explore the impact of window size, we run our models by increasing window size from 4 to 256 which exceeds the lengths of all sentences so that all timestamps are enclosed by the window.
The variation of MAE and F1 in this process is depicted in Fig.~\ref{fig:perf_by_window}. 
There is a dropping trend (MAE increment or F1 decrement) towards both sides of the optimal size.
We argue that it is because when the window expands, it is more probable for the newly included frame to add noise rather than provide valuable alignment information.
In the beginning, the marginal benefit is huge so the performance almost keeps climbing.
The optimal size is reached when the marginal benefit decreases to zero.

\begin{figure}[h]
\centering
\includegraphics[scale=0.44, trim=0cm 0cm 0cm 0cm]{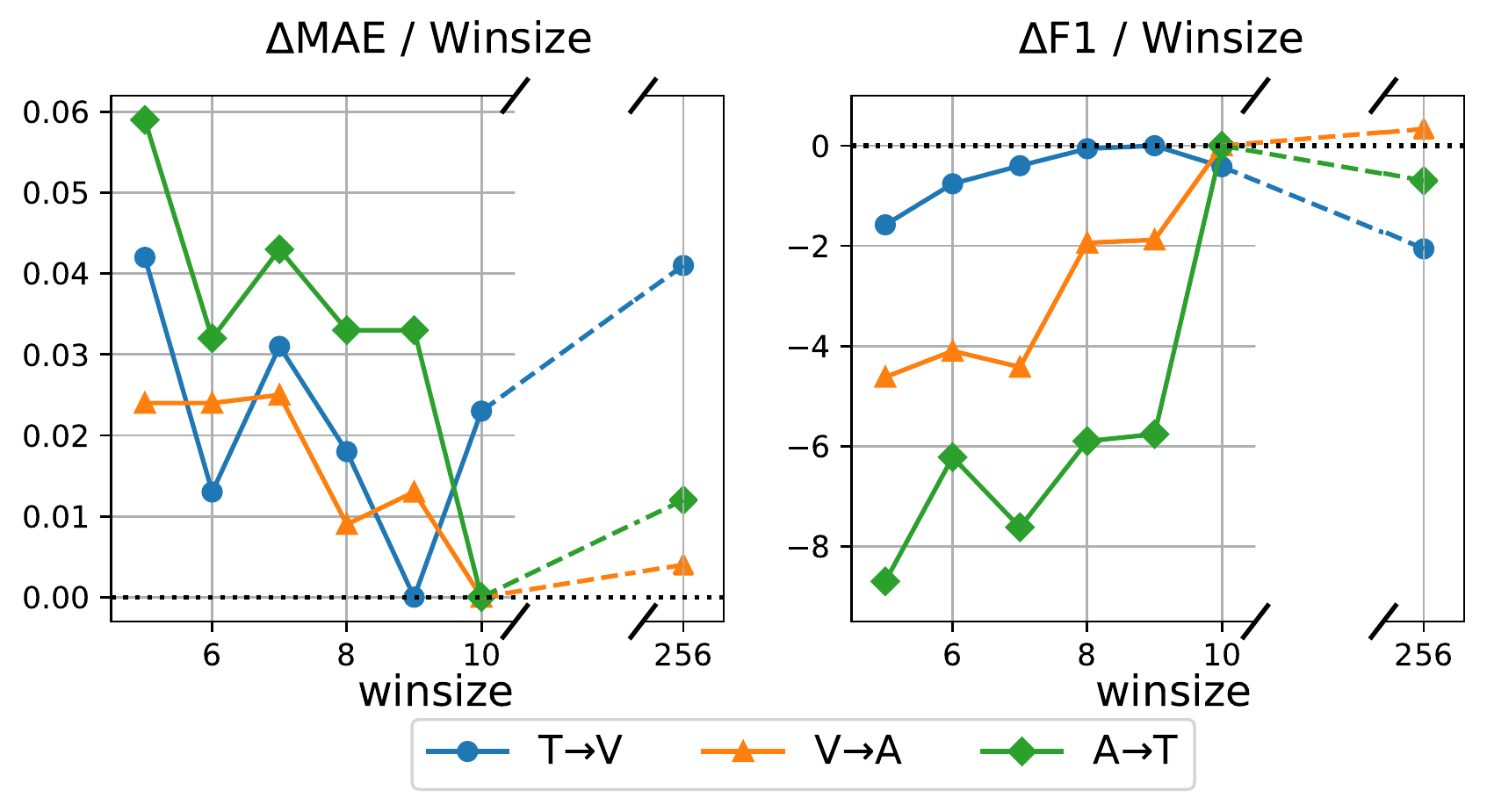}
\caption{Performance variation under different window sizes. The optimal sizes for the three pairs are 9, 10, 10.}
\label{fig:perf_by_window}
\end{figure}

To explain this claim, we randomly select a raw example from the CMU-MOSI dataset. 
As shown in Fig.~\ref{fig:case_study}, the textual expression does not advance in a uniform speed.
From the second to the third word 1.80 seconds elapses, while the last eight words are covered in only 2.53 seconds.
Intuitively we can assume all the frames in the video that span across the pronunciation of a word are \textit{causally correlated} with that word so that the representation mappings from the word to these frames are necessary and can benefit the downstream tasks.
For example, for the word ``I'' present at $t=1$ in text, it can benefit the timestamps until at least $t=5$ in the visual modality.
Note that we may overlook some potential advantages that could not be easily justified in this way and possess different effect scope, but we deem that those advantages would like-wisely disappear as the window size keeps growing.

\section{Conclusion}

In this paper, we propose~\modelname, a fast and efficient framework for the problem of missing modality inference.
It applies the theory of optimal transport to learn the alignment dynamics between temporal modality sequences for the inference in the case of missing modality sequences. 
Experiments on three datasets of demonstrate that~\modelname~can achieve much better performance and thus reveal the higher robustness of our method.
We hope that our work can inspire other research works in this field. 

\section*{Limitations}

Although our model has successfully tackled the two missing patterns, it may still fail in more complicated cases. 
For example, if missing happens randomly in terms of frames (some timestamps within a unimodal clip) instead of instances (the entire unimodal clip), then our proposed approach could not be directly used to deal with the problem, since we need at least several instances of complete parallel data to learn how to map from one modality sequences to the other. 
However, we believe these types of problems can still be properly solved by adding some mathematical tools like interpolation, etc. 
We will consider this idea as the direction of our future work.

Besides, the generalization capability of our framework on other multimodal tasks is not clear. 
But at least we know the feasibility highly depends on the types of target tasks, especially the input formats---they have to be parallel sequences so that temporal alignment information between these sequences can be utilized. The missing patterns should be similar to what we described in section 2, as we discussed in the first paragraph.
\section*{Acknowledgements}

This research is supported by the SRG grant id: T1SRIS19149 and the Ministry of Education, Singapore, under its AcRF Tier-2 grant (Project no. T2MOE2008, and Grantor reference no. MOET2EP20220-0017). Any opinions, findings, conclusions, or recommendations expressed in this material are those of the author(s) and do not reflect the views of the Ministry of Education, Singapore.

\bibliography{anthology,custom}
\bibliographystyle{acl_natbib}

\appendix
\section{Dataset Statistics and Preprocessing}
\label{sec:ds_stat}
The statistics of the two datasets are listed in Table~\ref{tab:data_statistics}. 
MELD is originally a dialogue emotion detection dataset, where each dialogue contains many sentences. 
Since we want to make it compatible with tested models, we extract all sentences and remove those that lack at least one modality (text, visual, acoustic).
Following previous work, for MOSI and MOSEI we use COVAREP~\cite{degottex2014covarep} and P2FA~\cite{yuan2008speaker} to respectively extract visual and acoustic features.
For MELD, we use ResNet-101~\cite{he2016deep} and Wave2Vec 2.0~\cite{baevski2020wav2vec} to extract visual and acoustic features.
\begin{table}[ht]
    \centering
    \begin{tabular}{c|c c c c}
    \toprule
     Dataset    &  Train & Dev & Test & Total \\
    \hline
    CMU-MOSI & 1284  & 229  & 686 & 2299 \\
    CMU-MOSEI & 16326 & 1871 & 4859 & 22856 \\
    MELD & 9988 & 1108 & 2610 & 13706 \\
    \bottomrule
    \end{tabular}
    \caption{Statistics of three datasets we use for experiments.}
    \label{tab:data_statistics}
\end{table}

\section{Hyperparameter Search}
\label{sec:hp_search}
All these models are trained on a single RTX A6000 GPU. We use Glove~\cite{pennington2014glove} 300d to initialize the embedding of all the tokens.
We perform a grid search for part of the hyperparameters as~Table~\ref{tab:hp_search}.
\begin{table}[h!]
    \resizebox{\linewidth}{!}{
    \begin{tabular}{c|c c c}
        \toprule
         HP-name  &  CMU-MOSI & CMU-MOSEI & MELD \\
         \hline
         $\eta_{main}$ & 1e-3,2e-3 & 1e-4 & 1e-3,1e-4\\
         $\eta_{fit}$ & 1e-4,5e-4,1e-3 & 2e-5,1e-4 & 5e-4,5e-5 \\
         attn\_dim & 32,40 & 32,40 & 32,64 \\
         num\_head & 4,8 & 4,8 & 4,8 \\
         $n_b$ & 32 & 32 & 32 \\
         warm-up & 1,2 & 1,2 & 1 \\
         patience & 10 & 5 & 5 \\
         $\lambda$ & 0.05,0.1 & 0.05,0.1 & 0.05,0.1 \\
         $K$ & 4,5,8,9,10 & 4,5,6,7,8 & 3,4,5,8 \\
        \bottomrule
    \end{tabular}
}
    \caption{The hyperparameter search for three datasets}
    \label{tab:hp_search}
\end{table}

\section{OT Solution}
\label{sec:sol_math}
\subsection{Visualization of Solutions}
\label{sec:align_vis}
\begin{figure}[ht]
    \centering
    \includegraphics[scale=0.42]{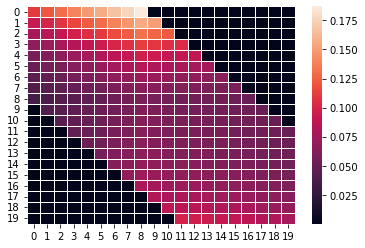}
    \caption{The average absolute entry values of the produced alignment matrices (window size=8).}
    \label{fig:avg_mat}
\end{figure}
To verify our statement in Section~\ref{sec:ADL} that the learned dynamics matrices are in the window style, we calculate and visualize the mean absolute values for each entry.
Due to various sentence lengths, the values are averaged over all matrices whose corresponding input sequences' lengths are no smaller than 20.
We visualize the heat map of the average entry values in Fig.~\ref{fig:avg_mat}.
It can be clearly viewed that the values outside the window stay nearly 0 (black squares), implying that they are always close to 0.

\subsection{Proof of solution pattern}
\label{sec:pattern_proof}
We formalize the window style solution in mathematical language.
\newtheorem{thm}{Theorem}
\begin{thm}
Given the optimal transport formulation as Eq.~\eqref{eq:goal}\textasciitilde~\eqref{eq:entropy_form}. All the entries $a^*_{ij}$ that satisfy $\vert i-j \vert > W$ in the optimal transport plan $A^*$ are 0, where $W$ is the window size. 
\end{thm}

\begin{proof}
We use the proof by contradiction. Assume there is an entry $A_{i'j'}$ in $A^*$ outside the window, i.e., $|i'-j'|>W$, and $A_{i'j'}>0$.
Then we have the cost $C=\sum A^*_{ij}M_{ij} \geq A_{i'j'}\times M_{i'j'} \to\infty$.
It is easy to find another path $i'\to k_1\to k_2\to\cdots\to k_n\to j'$, where $max(|i'-k_1|,|k_t-k_{t-1}|,|j'-k_n|)\leq W$.
In this new transport plan $A'$ simply we have $\sum A'_{ij}M_{ij} < \infty$, which means $A^*$ is not the optimal transport plan and contradicts our basic assumption.
Hence, by applying this kind of cost function we can obtain a window-style solution.
\end{proof}

\section{Complexity Analysis}
\label{sec:comp_analysis}
We conduct a simple analysis of the computational complexity of~\modelname~and~Modal-Trans.
We are concern about the stage that occupies the most time in one training epoch---training on the missing split when the ADL works in the decoding mode.
Suppose the average sequence length, the embedding dimension, the window size are $l$, $d$ and $w$ (here $w$ stands for the value of $2W+1$ for simplicity), respectively. 
The complexity (number of multiplication operations) of the alignment dynamics fitter is the summation of the complexity from GRU and the linear projection layer:
\begin{equation}
O(c_1ld^2)+O(c_2wld)\approx O(ld^2)
\end{equation}
The time spent on the alignment dynamics solver can be ignored since it is a non-parametric module so that no gradients are back-propagated through it and the number of iterations required for convergence is very little (about 5).
The complexity of the transformer decoder is the summation of the complexity from encoder-decoder attention, encoder \& decoder self-attention, and linear projections:
\begin{equation}
O(c_1l^3d)+O(c_2ld^2)+O(c_3l^2d) \approx O(l^3d) > O(ld^2)
\end{equation}
The last inequality is an empirical conclusion, since in our experiments $l\approx 10$ while $d=32$ in most situations.

Particularly, the complexity of encoder-decoder attention can be calculated by the summation of $l$ times individual attention in the decoding procedure:
\begin{equation}
    O(\sum_{i=1}^l (ild+ild)) = O((1+l)l\times ld) \approx O(l^3d)
\end{equation}
It should be highlighted that the computation only counts the number of multiplications into account.
Since sequence-to-sequence decoding can not be paralleled, it takes more time to train.

\section{Inference Speed}
\label{sec:inf_speed}
As we mentioned before, the most competitive baseline, Modal-Trans, is a variant of the most advanced sequence-to-sequence model. 
Apart from the performance improvement,~\modelname~also speeds up the training process.
To show this, we run and calculate the average batch training time between~\modelname~and Modal-Trans. As shown in Table~\ref{tab:speed},~\modelname~achieves over 3$\times$ training acceleration over Modal-Trans but can produce sequential imputation of higher quality.
We also provide an estimation for the computational complexity in the appendix.
\begin{table}[h]
    \small
    \centering
    \resizebox{\linewidth}{!}{
    \begin{tabular}{c|c|c|c}
    \toprule
     Model &  CMU-MOSI & CMU-MOSEI & MELD \\
     \midrule
     Modal-Trans & 0.811 & 1.270 & 0.954  \\
     \modelname~(window size=8) & 0.278 & 0.340 & 0.312 \\
     \bottomrule
    \end{tabular}
    }
    \caption{The average training time of the imputation module (seconds) per batch.}
    \label{tab:speed}
\end{table}

\section{Additional Results}
In the main text, we present the results of the minimum $p$ in both settings.
Here we also provide the results when tested in setting A for the two preservation in Table~\ref{table:mosi,appendix}, \ref{table:mosei,appendix} and~\ref{table:meld_appendix}.
\renewcommand{\arraystretch}{1.001}
\begin{table*}[ht]
    \centering
    \small
    \resizebox{\linewidth}{!}{
    \begin{tabular}{l*{3}{|c c | c c}}
        \toprule
        \multirow{3}{*}{Method} & 
        \multicolumn{4}{c|}{T $\rightarrow$ V} & \multicolumn{4}{c|}{V$\rightarrow$A}  & \multicolumn{4}{c}{A$\rightarrow$T}   \\
        \cline{2-13}
           ~ & \multicolumn{2}{c|}{10\%} & \multicolumn{2}{c|}{50\%} & \multicolumn{2}{c|}{10\%} & \multicolumn{2}{c|}{50\%} & \multicolumn{2}{c|}{10\%} & \multicolumn{2}{c}{50\%} \\
         \cline{2-13}
         ~ & \textbf{MAE$\downarrow$} & \textbf{Acc-2$\uparrow$} & \textbf{MAE$\downarrow$} & \textbf{Acc-2$\uparrow$} & \textbf{MAE$\downarrow$} & \textbf{Acc-2$\uparrow$} & \textbf{MAE$\downarrow$} & \textbf{Acc-2$\uparrow$} & \textbf{MAE$\downarrow$} & \textbf{Acc-2$\uparrow$} & \textbf{MAE$\downarrow$} & \textbf{Acc-2$\uparrow$} \\
        \midrule
        Supervised-Single (LB) & 1.242 & 68.6 & 1.242 & 68.6 & 1.442 & 46.4 & 1.442 & 46.4 & 1.440 & 42.2 & 1.440 & 42.2 \\
        Supervised-Both (UB) & 1.019 & 77.7 & 1.019 & 77.7 & 1.413 & 57.8 & 1.413 & 57.8 & 1.081 & 75.8 & 1.081 & 75.8 \\
        \cline{1-13}
        MFM & 1.103 & 71.0 & 1.098 & 73.1 & 1.456 & 43.5 & 1.471 & 42.2  & 1.477 & 42.2 & 1.451 & 42.7 \\
        SMIL & 1.073 & 74.2 & 1.060 & 75.0 & 1.442 & 45.9 & 1.471 & 42.7 & 1.447 & 43.3 & 1.473 & 45.3 \\
        Modal-Trans & 1.052 & 75.5 & 1.031 & 75.9 & 1.428 & 49.4 & 1.417 & 51.1 & 1.435 & 48.7 & 1.415 & 53.7 \\
        \cline{1-13}
        MM-Align (Ours) & \textbf{1.028} & \textbf{76.9} & \textbf{1.015} & \textbf{77.1} & \textbf{1.416} & \textbf{52.0}& \textbf{1.410} & \textbf{53.2} & \textbf{1.426} & \textbf{51.5} & \textbf{1.414} & \textbf{54.9}\\
        \midrule
        \multirow{3}{*}{} & 
        \multicolumn{4}{c|}{V $\rightarrow$ T} & \multicolumn{4}{c|}{A$\rightarrow$V}  & \multicolumn{4}{c}{T$\rightarrow$A}   \\
        \cline{2-13}
           ~ & \multicolumn{2}{c|}{10\%} & \multicolumn{2}{c|}{50\%} & \multicolumn{2}{c|}{10\%} & \multicolumn{2}{c|}{50\%} & \multicolumn{2}{c|}{10\%} & \multicolumn{2}{c}{50\%} \\
         \cline{2-13}
         ~ & \textbf{MAE$\downarrow$} & \textbf{Acc-2$\uparrow$} & \textbf{MAE$\downarrow$} & \textbf{Acc-2$\uparrow$} & \textbf{MAE$\downarrow$} & \textbf{Acc-2$\uparrow$} & \textbf{MAE$\downarrow$} & \textbf{Acc-2$\uparrow$} & \textbf{MAE$\downarrow$} & \textbf{Acc-2$\uparrow$} & \textbf{MAE$\downarrow$} & \textbf{Acc-2$\uparrow$} \\
        \midrule
        Supervised-Single (LB) & 1.442 & 46.4 & 1.442 & 46.4 & 1.440 & 42.2 & 1.440 & 42.2 & 1.242 & 68.6 & 1.242 & 68.6 \\
        Supervised-Both (UB) & 1.019 & 77.7 & 1.019 & 77.7 & 1.413 & 57.8 & 1.413 & 57.8 & 1.081 & 75.8  & 1.081 & 75.8 \\
        \cline{1-13}
        MFM & 1.446 & 45.5 & 1.429 & 48.3 & 1.454 & 42.2 & 1.467 & 42.2 & 1.078 & 72.9 & 1.083 & 73.3 \\
        SMIL & 1.448 & 44.2 & 1.461 & 46.1 & 1.442 & 45.9 & 1.441 & 46.4 & 1.060 & 75.5 & 1.091 & 74.9 \\
        Modal-Trans & 1.429 & 50.1 & \textbf{1.398} & 54.2 & 1.439 & 47.4 & 1.431 & 52.5 & 1.052 & 75.2 & 1.028 & \textbf{76.7} \\
        \cline{1-13}
        MM-Align (Ours) & \textbf{1.415} & \textbf{52.7} & 1.399 & \textbf{55.4} & \textbf{1.427}& \textbf{49.9} & \textbf{1.413} & \textbf{56.6} & \textbf{1.028} & \textbf{76.7} & \textbf{1.025} & \textbf{76.7} \\
        \bottomrule
    \end{tabular}
    }
    \caption{CMU-MOSI results in setting A (Fig.~\ref{ours}), where $p=10\%$ and $50\%$.}
    \label{table:mosi,appendix}
\end{table*}

\renewcommand{\arraystretch}{1.005}
\begin{table*}[ht]
    \centering
    \small
    \resizebox{\linewidth}{!}{
    \begin{tabular}{l*{3}{|c c | c c}}
        \toprule
        \multirow{3}{*}{Method} & 
        \multicolumn{4}{c|}{T $\rightarrow$ V} & \multicolumn{4}{c|}{V$\rightarrow$A}  & \multicolumn{4}{c}{A$\rightarrow$T}   \\
        \cline{2-13}
           ~ & \multicolumn{2}{c|}{10\%} & \multicolumn{2}{c|}{50\%} & \multicolumn{2}{c|}{10\%} & \multicolumn{2}{c|}{50\%} & \multicolumn{2}{c|}{10\%} & \multicolumn{2}{c}{50\%} \\
         \cline{2-13}
         ~ & \textbf{MAE$\downarrow$} & \textbf{Acc-2$\uparrow$} & \textbf{MAE$\downarrow$} & \textbf{Acc-2$\uparrow$} & \textbf{MAE$\downarrow$} & \textbf{Acc-2$\uparrow$} & \textbf{MAE$\downarrow$} & \textbf{Acc-2$\uparrow$} & \textbf{MAE$\downarrow$} & \textbf{Acc-2$\uparrow$} & \textbf{MAE$\downarrow$} & \textbf{Acc-2$\uparrow$} \\
        \midrule
        Supervised-Single (LB) & 0.687  & 77.4 & 0.687 & 77.4 & 0.836 & 61.3 & 0.836 &  61.3 & 0.851 & 62.9 & 0.851 & 62.9 \\
        Supervised-Both (UB) & 0.615 & 81.3 & 0.615 & 81.3 & 0.707 & 79.5 & 0.707 & 79.5 & 0.613 & 80.9 & 0.613 & 80.9 \\
        \cline{1-13}
        MFM & 0.658 & 79.2 & 0.641 & 78.7 & 0.827 & 60.7 & 0.816 & 62.4 & 0.830 & 64.5 & 0.836 & 63.5 \\
        SMIL & 0.680 & 78.3 & 0.654 & 78.5 & 0.819 & 64.3 & 0.815  & 64.6 & 0.840 & 62.9 & 0.835 & 63.5 \\
        Modal-Trans & 0.645 & 79.6 & 0.641 & 79.5 & 0.818 & 64.7 & 0.814 & 64.7 & 0.827 & 64.9 & 0.820 & 64.7 \\
        \cline{1-13}
        MM-Align (Ours) & \textbf{0.637} & \textbf{80.8} & \textbf{0.623} & \textbf{81.0} & \textbf{0.811} & \textbf{65.9} & \textbf{0.808} & \textbf{66.1} & \textbf{0.824} & \textbf{65.3} & \textbf{0.817} & \textbf{65.7} \\
        \midrule
        \multirow{3}{*}{} &  
        \multicolumn{4}{c|}{V $\rightarrow$ T} & \multicolumn{4}{c|}{A$\rightarrow$V}  & \multicolumn{4}{c}{T$\rightarrow$A}   \\
        \cline{2-13}
           ~ & \multicolumn{2}{c|}{10\%} & \multicolumn{2}{c|}{50\%} & \multicolumn{2}{c|}{10\%} & \multicolumn{2}{c|}{50\%} & \multicolumn{2}{c|}{10\%} & \multicolumn{2}{c}{50\%} \\
         \cline{2-13}
         ~ & \textbf{MAE$\downarrow$} & \textbf{Acc-2$\uparrow$} & \textbf{MAE$\downarrow$} & \textbf{Acc-2$\uparrow$} & \textbf{MAE$\downarrow$} & \textbf{Acc-2$\uparrow$} & \textbf{MAE$\downarrow$} & \textbf{Acc-2$\uparrow$} & \textbf{MAE$\downarrow$} & \textbf{Acc-2$\uparrow$} & \textbf{MAE$\downarrow$} & \textbf{Acc-2$\uparrow$} \\
        \midrule
        Supervised-Single (LB) & 0.836 & 61.3 & 0.836 & 61.3 & 0.851 & 62.9 & 0.851 & 62.9 & 0.687 & 77.4 & 0.687 & 77.4 \\
        Supervised-Both (UB) & 0.615 & 81.3 & 0.615 & 81.3 & 0.707 & 79.5 & 0.707 & 79.5 & 0.613 & 80.9 & 0.613 & 80.9 \\
        \cline{1-13}
        MFM & 0.821 & 62.0 & 0.820 & 64.5 & 0.842 & 62.7 & 0.835 & 62.4 & 0.658 & 79.1 & 0.659 & 78.9 \\
        SMIL & 0.820 & 63.1 & 0.817 & 63.5 & 0.838 & 63.2 & 0.829 & 64.2 & 0.684 & 78.5 & 0.658 & 79.4  \\
        Modal-Trans & 0.817 & 64.9 & 0.815 & 64.9 & 0.832 & 64.6 & 0.825 & 64.7 & 0.643 & 79.9 & 0.648 & 79.7 \\
        \cline{1-13}
        MM-Align (Ours) & \textbf{0.812} & \textbf{65.2} & \textbf{0.807} & \textbf{66.9} & \textbf{0.822} & \textbf{65.4} & \textbf{0.819} & \textbf{66.0} & \textbf{0.635} & \textbf{81.0} & \textbf{0.626} & \textbf{80.9} \\
        \bottomrule
    \end{tabular}
    }
    \caption{CMU-MOSEI results in Setting A ($p=10\%$ and $50\%$).}
    \label{table:mosei,appendix}
\end{table*}

\renewcommand{\arraystretch}{1.02}
\begin{table}[h]
    \centering
    \small
    \resizebox{\linewidth}{!}{
    \begin{tabular}{l*{3}{|c|c}}
        \toprule
        \multirow{2}{*}{Method} &
         10\% & 50\% &  10\% & 50\% & 10\% & 50\% \\
        \cline{2-7}
        ~ & \multicolumn{2}{c|}{T $\rightarrow$ V} & \multicolumn{2}{c|}{V$\rightarrow$A}  & \multicolumn{2}{c}{A$\rightarrow$T}   \\
        \midrule
        Supervised-Single (LB) & 54.0 & 54.0 & 31.3 & 31.3 & 31.3 & 31.3  \\
        Supervised-Both (UB)  & 55.8 & 55.8 & 32.1 & 32.1 & 55.9 & 55.9 \\
        \cline{1-7}
        MFM    & 54.0 & 54.0 & 31.3 & 31.3 & 31.3 & 43.1 \\
        SMIL   & 54.1 & 54.4 & 31.3 & 31.3 & 31.3 & 43.5\\
        Modal-Trans & 54.2 & 55.0 & 31.3 & 31.4 & 31.5 & 44.4 \\
        \cline{1-7}
        MM-Align (Ours) & \textbf{54.2} & \textbf{55.7} & 31.3 & \textbf{31.9} & 31.5 & \textbf{45.5}  \\
        \cline{1-7}
        ~ & \multicolumn{2}{c|}{V $\rightarrow$ T} & \multicolumn{2}{c|}{A $\rightarrow$V}  & \multicolumn{2}{c}{T $\rightarrow$A}   \\
        \midrule
        Supervised-Single (LB) & 31.3 & 31.3 & 31.3 & 31.3 & 54.0 & 54.0 \\
        Supervised-Both (UB)  & 55.8 & 55.8 & 32.1 & 32.1 & 55.9 & 55.9  \\
        \cline{1-7}
        MFM & 31.4 & 43.6 & 31.3 & 31.3 & 54.2 & 54.1   \\
        SMIL & 31.4 & 43.9 & 31.3 & 31.3 & 54.5 & 54.2   \\
        Modal-Trans & 31.6 & 44.2 & 31.3 & 31.3 & 55.0 & 54.8  \\
        \cline{1-7}
        MM-Align (Ours) & \textbf{32.3} & \textbf{45.4} & 31.3 & \textbf{32.0} & \textbf{55.6} & \textbf{55.7}  \\
        \bottomrule
    \end{tabular}
    }
    \caption{Results on MELD~($p=10\%$ and $50\%$) in setting A.}
    \label{table:meld_appendix}
\end{table}
\end{document}